\documentclass{article}

\usepackage{microtype}
\usepackage{graphicx}
\usepackage{subfigure}
\usepackage{booktabs} 
\usepackage{hyperref}

\usepackage[accepted]{icml2022}

\usepackage{amsmath}
\usepackage{amssymb}
\usepackage{mathtools}
\usepackage{amsthm}
\usepackage{graphicx}
\usepackage{subfigure}
\usepackage{booktabs} 
\usepackage{bm}
\usepackage{physics}
\usepackage{float}
\usepackage{caption}
\usepackage{enumitem}
\usepackage{appendix}
\usepackage{amssymb}
\usepackage{amsthm}
\usepackage{amstext}
\usepackage{amsmath}
\usepackage{amsfonts}
\usepackage{paralist}
\usepackage{pdfpages}
\usepackage{makecell}
\usepackage{dsfont}
\usepackage{hyperref}
\newtheorem{theorem}{Theorem}[section]
\usepackage[capitalize,noabbrev]{cleveref}
\theoremstyle{plain}

\newtheorem{lemma}[theorem]{Lemma}

\theoremstyle{definition}

\theoremstyle{remark}

\usepackage[textsize=tiny]{todonotes}

\icmltitlerunning{3DLinker: An E(3) Equivariant Variational Autoencoder for Molecular Linker Design}

\begin{document}

\twocolumn[
\icmltitle{3DLinker: An E(3) Equivariant Variational Autoencoder\\ for Molecular Linker Design}



\icmlsetsymbol{equal}{*}

\begin{icmlauthorlist}
\icmlauthor{Yinan Huang}{bigai}
\icmlauthor{Xingang Peng}{thu}
\icmlauthor{Jianzhu Ma}{pku,bigai}
\icmlauthor{Muhan Zhang}{pku,bigai}
\end{icmlauthorlist}

\icmlaffiliation{bigai}{Beijing Institute for General Artificial Intelligence}
\icmlaffiliation{pku}{Institute for Artificial Intelligence, Peking University}
\icmlaffiliation{thu}{Tsinghua University}
\icmlcorrespondingauthor{Muhan Zhang}{muhan@pku.edu.cn}
\icmlcorrespondingauthor{Jianzhu Ma}{majianzhu@pku.edu.cn}

\icmlkeywords{Machine Learning, ICML}

\vskip 0.3in
]



\printAffiliationsAndNotice{}  

\begin{abstract}


Deep learning has achieved tremendous success in designing novel chemical compounds with desirable pharmaceutical properties. In this work, we focus on a new type of drug design problem---generating a small ``linker'' to physically attach two independent molecules with their distinct functions. The main computational challenges include: 1) the generation of linkers is \textit{conditional} on the two given molecules, in contrast to generating full molecules from scratch in previous works; 2) linkers heavily depend on the anchor atoms of the two molecules to be connected, which are not known beforehand; 3) 3D structures and orientations of the molecules need to be considered to avoid atom clashes, for which equivariance to E(3) group are necessary. To address these problems, we propose a conditional generative model, named 3DLinker, which is able to predict anchor atoms and jointly generate linker graphs and their 3D structures based on an E(3) equivariant graph variational autoencoder. So far as we know, there are no previous models that could achieve this task. We compare our model with multiple conditional generative models modified from other molecular design tasks and find that our model has a significantly higher rate in recovering molecular graphs, and more importantly, accurately predicting the 3D coordinates of all the atoms.
\end{abstract}

\section{Introduction}
\label{introduction}

The biological functions of most small molecule drugs are to inhibit the activity of the target protein by binding its active sites. In drug discover, designing new molecule drugs with desired pharmacophoric properties remains challenging due to the discreteness and enormity of the search space~\citep{polishchuk2013estimation}. To address this problem, many machine learning methods have been developed to embed molecules to a compact hidden space and hence make promising progresses in multiple downstream computational tasks such as molecular de-novo design, molecular optimization, and chemical property prediction. 

Molecules are generally represented by graphs with atoms and bonds represented as nodes and edges, respectively. Graph generative models~\citep{liu2018constrained, shi2019graphaf, jin2018junction, jin2020hierarchical} are commonly applied to model the marginal probability for FDA-approved drug molecules and it is expected that the newly sampled molecules from the model have similar or better pharmacophoric properties.

However, in complex diseases such as cancer, mutations of amino acids could significantly impact the binding affinity between drugs and target proteins. The drug might fall off the drug target when a particular amino acid mutates with a certain probability due to the weak binding affinity, which makes the patient drug resistant. To solve this problem, more recently, an alternative drug mechanism named Proteolysis targeting chimera (PROTAC) is developed to inhibit the protein functions by prompting complete degradation of the target protein. PROTAC is a unique molecule composed of two \textit{fragment} molecules and a \textit{linker} molecule: one fragment binds the target protein, the other fragment binds another molecule that can degrade the target protein, and the linker attaches the two fragments together. Because PROTAC needs only to bind their targets with high selectivity (rather than inhibit the target protein's activity), many efforts are devoted to retool previously ineffective inhibitor molecules as PROTAC for developing the next-generation drugs. Even though PROTAC owns promising potential, it has not been broadly pushed into clinical trial stages. One of the key challenges is the design of linker, which has critical influence on the ultimate degradation of target protein. To date, linker design still relies on the expertise of structural biologists and thus is very time intensive. Therefore, there are increasing efforts to develop deep learning methods to address linker design problems~\citep{imrie2020deep, yang2020syntalinker}.




A critical challenge for computational linker design stems from its strong 3D spatial constraints compared to classical graph generation tasks. It is known that a successful fragment linker should not disturb the spatial configurations of the two fragments~\citep{ichihara2011compound, klon2015fragment}. In addition, the anchors between fragments and linker also have correlations with their spatial poses. The linker design problem should be extended to include the 3D information, and a 3D-aware generative model is in need for the generation of realistic linkers, rather than invoke a graph generative model. In this paper, we propose a conditional generative model, named 3DLinker, that jointly models the 2D molecular graphs and 3D structures of linker for solving the 3D linker design problem. 

Given the graph and spatial coordinates of two fragments, 3DLinker can jointly generate graph and spatial coordinates of the linker. Particularly, it does not rely on pre-determined anchors and can accurately predict anchors based on the two observed fragments to be connected. More importantly, 3DLinker is able to predict 3D coordinates directly and at the same time keeps equivariant to rotations, translations and reflections, which makes it insensitive to choices of the coordinate system. Finally, since the generative model is based on the variational autoencoder (VAE) framework~\citep{kingma2013auto}, it can be used as an unsupervised representation learning method whose latent representations are fed into downstream tasks such as drug-likeness prediction. To the best of our knowledge, 3DLinker is the first trial that simultaneously predicts equivariant graph and 3D coordinates for the linker design problem.

\begin{figure}[t]
    \centering
    \includegraphics[page=1, scale=0.4]{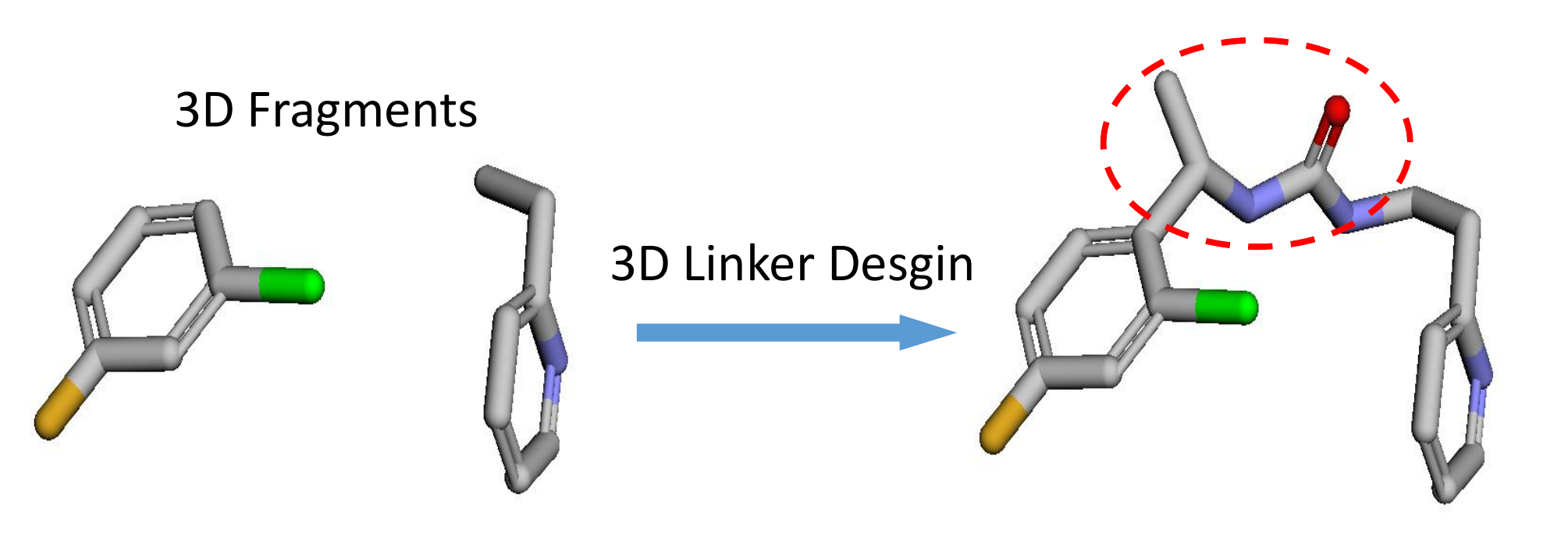}
    \vspace{-10pt}
    \caption{3D linker design problem: given two fragments' graph with 3D coordinates (left), the goal is to generate a linker graph with 3D coordinates to link these two fragments (right). The 3D coordinates of the generated linker must align with the two fragments, otherwise they cannot link.}
    \label{3D Linker Design}
\end{figure}

\section{Background}
\label{background}
In this section we introduce the definition of the 3D linker design problem and basic concepts of E(3) equivariance.
\subsection{3D Linker Design}
\label{background-3d_linker_design}
Linker design is to generate a small molecule that can link two given molecular fragments at certain anchors (binding atoms). Instead of modelling it as a 2D graph generation problem, it is important to take 3D information into account, since the designed fragment should satisfy spatial constraints such as not disturbing the relative poses or causing any atom clashes of the two fragments. Therefore, it requires the linker design algorithm to be able to generate both the chemical graph and 3D coordinates given the two fragment and their spatial coordinates (Figure~\ref{3D Linker Design}).

Mathematically, a molecule can be viewed as a graph with atoms as nodes and chemical bonds as edges. A 3D molecule can be represented by a graph $G=(\mathcal{V}, \mathcal{E}, X)$ with 3D coordinates $\bm{r}=(x, y, z)$, where $\mathcal{V}$ is the set of nodes, $\abs{\mathcal{V}}$ is the number of nodes, $\mathcal{E} \subset \mathcal{V} \times \mathcal{V}$ are edges, $X$ are node types and $\bm{R}$ is a matrix whose $i$-th row is $\bm{r}_i^{\top}$ (${\top}$ stands for transpose). In the 3D linker design, two fragments are defined by two unlinked subgraphs $G_{\text{F}}=(G_{\text{F},\text{1}}, G_{\text{F},\text{2}})$ with geometry $\bm{R}_{\text{F}}$, and a linker is denoted by $G_{\text{L}}$ with geometry $\bm{R}_{\text{L}}$. Let $G,\bm{R}$ be the graph and geometry of the ground truth linked molecule containing both fragments and linker. A 3D linker design model is a conditional generative model that completes the ground truth molecule graph as well as its geometry given the two fragments:
\begin{equation}
    p(G,\bm{R}|G_{\text{F}},\bm{R}_{\text{F}}).
\end{equation}

\subsection{O(3), E(3) Groups and Equivariance}
Group is a set of operations equipped with multiplications, associativity, the identity element and the inverse element. Group of all 3D rotations and reflections is called 3D Orthogonal Group or O(3), and group of all 3D rotations, translations and reflections is called 3D Euclidean Group or E(3). Let $\mathcal{X}$ be the input space, $\mathcal{Y}$ be output space and $\text{GL}(\mathcal{X})$ be all invertible linear transformations from $\mathcal{X}$ to $\mathcal{X}$ (similar for $\text{GL}(\mathcal{Y})$). A function $\phi:\mathcal{X}\to\mathcal{Y}$ is called \textit{equivariant} to the group $\mathcal{G}$, if for all group element $g\in\mathcal{G}$ and all $x\in\mathcal{X}$, there exists group representations $\pi^{\mathcal{X}}: \mathcal{G}\to \text{GL}(\mathcal{X})$ and $\pi^{\mathcal{Y}}: \mathcal{G}\to \text{GL}(\mathcal{Y})$ such that 
\begin{equation}
    \pi^{\mathcal{Y}}(g)\phi(x) = \phi(\pi^{\mathcal{X}}(g)x).
\end{equation}
If $\pi^{\mathcal{Y}}(g)$ is the identity function for all $g\in\mathcal{G}$, then we say $\phi$ is \textit{invariant} to group $\mathcal{G}$. 
In 3D linker design, it is known that a molecule graph $G$ should not depend on a specific coordinate system and $\bm{R}$ should change equivariantly to transformations of the coordinate system. Therefore it raises a constraint that for any $g\in E(3)$, the generative model $p(G, \bm{R}|G_{\text{F}},\bm{R}_{\text{F}})$ should satisfy
\begin{equation}
 p(G, \pi(g)\bm{R}|G_{\text{F}},\pi(g)\bm{R}_{\text{F}}) = p(G, \bm{R}|G_{\text{F}},\bm{R}_{\text{F}}),
 \label{eqv-condition}
\end{equation}
where $\pi(g)$ can by any rotation, translations or reflections matrix in the 3D space. 

\section{Related Work}
\label{related work}



\textbf{Graph-based Molecular Generation.} Variational Graph Auto-Encoders are the most popular models for molecular generation. Early approaches mainly focus on embedding the 2D chemical graphs into low dimensional space and sample new molecules by perturbing the hidden values. The representative works include GraphVAE~\citep{simonovsky2018graphvae}, CGVAE~\citep{liu2018constrained}, JT-VAE~\citep{jin2018junction}, GraphNVP ~\citep{madhawa2019graphnvp} and so on. Although none of these methods designs molecular linkers, graph-based VAE models serve as the basic building block of our entire architecture and all of these models could be plug into our framework by transforming the generative model to conditional generative model. For the perspective of training techniques, auto-regression are also widely adopted to train graph-based deep learning models, such as GraphRNN~\citep{you2018graphrnn}, DeepGMG~\citep{li2018learning}, GraphAF~\citep{shi2019graphaf}. Most of these model generate nodes and edges in a sequential manner.


\textbf{Point-cloud-based Molecular Models.} An important component of our work is to model and design the 3D structures of molecular fragments, in which maintaining the equivariant properties is the crucial computational challenge. One typical solution is to model the molecules as 3D point clouds using equivariant neural networks~\citep{schutt2017schnet, klicpera2020directional, liu2021spherical, satorras2021n, thomas2018tensor, fuchs2020se, deng2021vector, jing2020learning}. To train such models, auto-regression is a more common solution, such as G-SchNet~\citep{gebauer2019symmetry}, G-SphereNet~\citep{anonymous2022an}, but flow-based model ENF~\citep{satorras2021enf} and reinforcement learning ~\cite{simm2020reinforcement, simm2020symmetry}.could also be applied. The main limitation of point-cloud-based model is that they cannot directly generate discrete graph structures, which make it difficult to model chemical constraints like valency (maximal number of hydrogen atoms one can combine with).


\textbf{Molecular Linker Design.} DeLinker~\citep{imrie2020deep} is the first attempt to apply deep learning methods to the linker design problem. It constructs a conditional graph generative model that generates linker given two fragments. It adapts CGAVE~\citep{liu2018constrained}, generating edges step by step starting with fragments and two known anchor nodes as the binding sites. The spatial distance and angle between two fragments are provided to the model as side information to guide the generation. DEVELOP~\citep{imrie2021deep} improves DeLinker by encoding the spatial information of fragments using CNN. However, both of them only take (coarse) 3D information as input and do not have a precise atom-level description of molecule geometry, which are not sufficient to express the fragments geometry. In addition, anchor nodes are assumed to be known in advance, which is rare in real world application. Most importantly, they are only able to generate graph representations of linker without 3D coordinates.

\section{Methodology}
\label{method}
In this section, we present our 3DLinker, a conditional VAE-based generative model that generates both invariant graphs and equivariant absolute coordinates of linkers given two 3D fragments.  


 \textbf{Notations.} Let $G_{\text{F}}, G_{\text{L}}, G$ be graphs of fragments, linker and full molecule (ground truth) respectively, and similarly for coordinates $\bm{R}_{\text{F}}, \bm{R}_{\text{L}}, \bm{R}$ as in section \ref{background-3d_linker_design} . As we complete the full molecule graph step by step, we use $G_t$ and $\bm{R}_t$ to denote the current (existing) graph and coordinates at timestamp $t$, where $G_0 = G_\text{F}, \bm{R}_0 = \bm{R}_{\text{F}}$. The encoder embeds each node $i\in \mathcal{V}$ with both invariant features $h_i\in \mathbb{R}^{n_h}$ (for embedding the graph) and equivariant features $\bm{v}_i\in \mathbb{R}^{n_v\times 3}$ (for embedding the coordinates), which are further used for sampling invariant latent variables $z^h_i\in\mathbb{R}^{m_h}$ and equivariant latent variables $\bm{z^v}_i\in\mathbb{R}^{m_v\times 3}$. Symbols $h,\bm{v},z^h,\bm{z^v}$ without subscripts refer to that variable for all nodes in a general sense.  For column vectors $a\in\mathbb{R}^c$ and $b\in\mathbb{R}^c$, we use $a\odot b\in\mathbb{R}^c$ to denote point-wise multiplication and $\text{diag}\{a\}\in\mathbb{R}^{c\times c}$ to denote a matrix whose diagonal is $a$ and zero otherwise. 

\textbf{Equivariant Features For Coordinates Predictions.} The equivariant nature makes it difficult to predict absolute coordinates directly. Many of existing works~\citep{gebauer2019symmetry, anonymous2022an, xu2021end} tackle this problem by encoding coordinate information as invariant node features and predicting invariant quantities such as distances and edge angles. However, these indirect methods are either computationally intensive (need transformation to local coordinate system~\citep{anonymous2022an}) or introducing extra error from the second nonconvex optimization (for translating distance matrices into absolute coordinates~\citep{xu2021end}). Instead, we propose to directly generate absolute coordinates while preserving equivariance. 
To generate equivariant coordinates, only leveraging invariant features is not enough: we cannot produce an equivariant quantity arbitrarily by combining invariant quantities. Therefore, in addition to invariant node features, we need to introduce extra equivariant node features that can be directly used for composing equivariant coordinates. We use notations $h$ for invariant features and $\bm{v}$ for equivariant features. 

\textbf{Vector Neurons.} Classical fully connected neural networks or MLPs cannot preserve equivariance and thus is not suitable for transforming equivariant features $\bm{v}$. In this regard, vector neuron networks or VN-MLP~\citep{deng2021vector} propose a ReLU-like nonlinear function for equivariant features. Concretely, given an equivariant input $\bm{v}\in \mathbb{R}^{n_v\times 3}$, Vector-ReLU learns two weight matrices $W\in\mathbb{R}^{n_v^{\prime}\times n_v}$ and $U\in\mathbb{R}^{n_v^{\prime}\times n_v}$ to map $\bm{v}$ to output $\bm{v}^{\prime}\in\mathbb{R}^{n_v^{\prime}\times 3}$ via
\begin{subequations}
    \begin{align}
  \bm{q} & = W\cdot\bm{v}\in\mathbb{R}^{n_v^{\prime}\times 3}, \quad \bm{k} = U\cdot\bm{v}\in\mathbb{R}^{n_v^{\prime}\times 3},\label{vector neuron eqn1}\\
  \bm{v}^{\prime}& = \bm{q} -\text{diag}\left\{\mathds{1}_{\langle \bm{q},\bm{k}\rangle<0} \odot\langle \bm{q},\frac{\bm{k}}{\norm{\bm{k}}}\rangle\right\}\cdot \frac{\bm{k}}{\norm{\bm{k}}},\label{vector neuron eqn2}
\end{align}
\label{vector neuron}%
\end{subequations}
where $\langle \bm{q},\bm{k}\rangle\in\mathbb{R}^{n_v^{\prime}}$ is the inner product in the last axis, $\mathds{1}$ is the indicator function, and $\norm{\bm{k}}\in\mathbb{R}^{n_v^{\prime}}$ is the norm of $\bm{k}$ over the last axis. It is easy to verify that $\bm{v}^{\prime}$ is equivariant, since both $\bm{q}$ and $\bm{k}$ are linear combinations of equivariant input $\bm{v}$ while coefficients $\langle \bm{q},\bm{k}\rangle$ is invariant. Intuitively, Vector-ReLU projects $\bm{q}$ to the orthogonal plane of a learnable direction $\bm{k}$ if $\bm{q}$ lies in the other side of the plane, which is analogous to the cutoff in classic ReLU. This nonlinearity enhances the expressive power while preserving the equivariance. We use $\text{VN-MLP}$ to denote a neural network stacked by multiple Vector-ReLU units. 

\textbf{Mixed-Features Message Passing.} Now we are ready to introduce our Mixed-Features Message Passing (MF-MP) scheme. MF-MP performs message passing for invariant features $h$ and equivariant features $\bm{v}$ simultaneously, and in each step the two types of features are properly mixed so that 1) their respective invariance and equivariance properties are preserved, and 2) one type of feature helps the update of the other type and vice versa. 

In the first step, invariant features $h\in\mathbb{R}^{n_h}$ and equivariant features $\bm{v}\in\mathbb{R}^{n_v\times 3}$ are transformed and mixed to construct new expressive intermediate features $h^{\prime}, h^{\prime\prime}, \bm{v}^{\prime}$ by 
\begin{subequations}
\begin{align}
h^{\prime}_j &= \phi_{1}(h_j, \norm{\text{VN-MLP}_{1}(\bm{v_j})})\in\mathbb{R}^{n_h}, \label{eqn-a}\\ 
h^{\prime\prime}_j&=\phi_{2}(h_j, \norm{\text{VN-MLP}_{2}(\bm{v_j})})\in\mathbb{R}^{n_v},\label{eqn-b}\\
\bm{v}^{\prime}_j &=\text{diag}\{\phi_{3}(h_j)\}\cdot\text{VN-MLP}_{3}(\bm{v}_j)\in \mathbb{R}^{n_v\times 3}.\label{eqn-c}
\end{align}
\label{mixed features}%
\end{subequations}
Next, point convolution~\citep{thomas2018tensor, schutt2017schnet, schutt2021equivariant} is applied to linearly transform the mixed features $h^{\prime},h^{\prime\prime},\bm{v}^{\prime}$ into messages:
\begin{subequations}
\begin{align}
 m^h_{i\leftarrow j} &= \text{Ker}_{1}(\norm{\bm{r}_{i,j}})\odot h^{\prime}_j,\label{eqn-d} \\
\bm{m^v}_{i\leftarrow j}&=  \text{diag}\left\{\text{Ker}_{2}(\norm{\bm{r}_{i,j}})\right\}\cdot\bm{v}^{\prime}_j  \nonumber\\
& +\big(\text{Ker}_{3}(\norm{\bm{r}_{i,j}})\odot h^{\prime\prime}_{j}\big)\cdot\bm{r}^{\top}_{i,j},
\label{eqn-e} 
\end{align}
\label{message function}%
\end{subequations}
where $\bm{r}_{i,j}=\bm{r}_i-\bm{r}_j$ is the relative displacement, $\text{Ker}$ are learnable kernels such as RBFs that 
transform a scalar distance into a multi-dimensional output vector using different shape parameters, making the messages geometry-aware. Intuitively, it reflects discrete levels of physical interactions (short-range, long-range) in different distances. More details on Ker are given in Appendix \ref{point convolution}.

Finally, Gated Recurrent Units (GRU)~\citep{li2015gated} and VN-MLP are applied as powerful nonlinear transformations to update the node features with the messages:
\begin{subequations}
\begin{align}
\tilde{h}_{i}& = \text{GRU}(h_i, \sum_{j\in N(i)}m^h_{i\leftarrow j}), \label{agg-h}\\
\tilde{\bm{v}}_{i} &= \text{VN-MLP}_{4}(\bm{v}_i, \sum_{j\in N(i)}\bm{m^v}_{i\leftarrow j}).\label{agg-v}
\end{align}
\label{aggregation function}%
\end{subequations}
Here $N(i)$ stands for neighbors of $i$. Our Mixed-Features Message Passing (MF-MP) above effectively mixes invariant and equivariant features in each step to help the update of each other with powerful nonlinear functions. Proof of MF-MP's equivariance w.r.t. E(3) is included in Appendix~\ref{appendix equivariance}. We also discuss how it is related and different from Tensor Field Networks~\citep{thomas2018tensor} in appendix \ref{point convolution}. 

Now we describe details about the encoder and decoder of 3DLinker using MF-MP as building blocks. 3DLinker is a conditional latent generative model $p_{\theta,\psi}(G,\bm{R}|G_{\text{F}},\bm{R}_{\text{F}})$ including an encoder $q_{\psi}(z^h, \bm{z^{v}}|G_{\text{F}}, G, \bm{R}_{\text{F}}, \bm{R})$, a decoder $p_{\theta}(G, \bm{R}|G_{\text{F}},\bm{R}_{\text{F}},z^h,\bm{z}^{\bm{v}})$ and a prior $p(z^h,\bm{z}^{\bm{v}}|G_{\text{F}},\bm{R}_{\text{F}})$.
\begin{figure*}[ht!]
  \centering
  \includegraphics[page=1, scale=0.8]{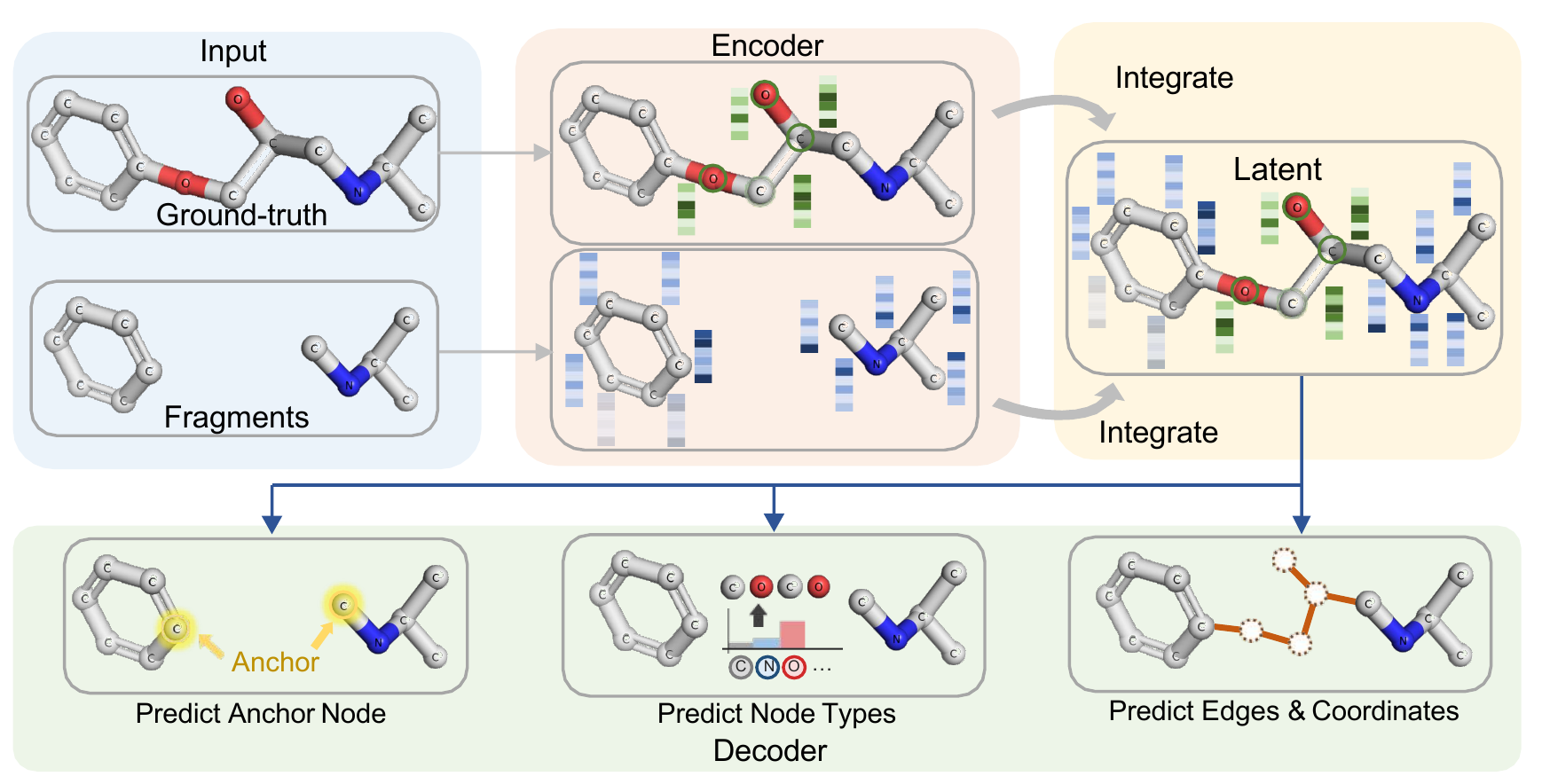}
  \vspace{-5pt}
  \caption{Illustration of overall encoding and decoding process. For encoding, ground truth is sent into a MF-MP encoder to get node-level representations. Those representations of nodes in fragments are discarded and replaced by representations that are computed separately on fragments graph only. For decoding, two anchor nodes are predicted as the binding sites for linker. Node Types of linker are simultaneously predicted before linking. With two anchor nodes and node types of linker, edges and coordinates are sequentially predicted, as demonstrated in Figure \ref{decoder}.}
  \label{vae}
\end{figure*}

\subsection{Encoder}
The encoder $q_{\psi}(z^h, \bm{z}^{\bm{v}}|G_{\text{F}}, G, \bm{R}_{\text{F}}, \bm{R})$ computes node-level latent distributions utilizing MF-MP. Initially, invariant features $h$ are embeddings of node types and we let equivariant features $\bm{v}=0$. After applying several times of MF-MP, we obtain the final node features $\tilde{h}$ and $\tilde{\bm{v}}$. The latent variables are sampled by $z^h_i{\sim} N(\mu_i^h, (\sigma_i^h)^2I)$, $\bm{z^v}_i{\sim} N(\bm{\mu^v}_i, (\sigma^{\bm{v}}_i)^2I)$ for linker nodes only, where the means and variances are computed from the final node features:
\begin{subequations}
\begin{align}
&\forall i\in \mathcal{V}_{\text{L}}, \nonumber \\
&\mu^h_i = \phi_{4}(\tilde{h}_i), \quad (\sigma^h_i)^2= \phi_{5}(\tilde{h}_i), \\
& \bm{\mu^v}_i=\text{VN-MLP}_{5}(\tilde{\bm{v}}_i),\quad (\sigma^{\bm{v}}_i)^2 = \phi_{6}(\tilde{h}_i).
\end{align}
\label{encoder}%
\end{subequations}
Note that for equivariant latent variables $\bm{z^v}$ the covariance $(\sigma^{\bm{v}})^2I$ assign the same variance to $x,y,z$ directions, which is the simplest way to preserve equivariance. Since fragments $(G_{\text{F}}, \bm{R}_{\text{F}})$ are given during generation, there is no need to sample their latent variables. Instead, we run the same (weight-sharing) MF-MP network again on fragments only, which gives another set of final node features $\hat{h}_i,\hat{\bm{v}}_i$ for fragment nodes. Latent variables of fragment nodes are deterministically obtained by
\begin{subequations}
\begin{align}
&\forall i\in \mathcal{V}_{\text{F}},\nonumber\\
&z^h_i =\phi_{7}(\hat{h}_i), \quad \bm{z^v}_i= \text{VN-MLP}_{6}(\hat{\bm{v}}_i).
\end{align}
\label{encoder-frag}%
\end{subequations}
\vspace{-20pt}
\subsection{Decoder}
The decoder $p_{\theta}(G, \bm{R}|G_{\text{F}},\bm{R}_{\text{F}}, z^h,\bm{z^v})$ constructs $(G,\bm{R})$ from fragments $(G_{\text{F}},\bm{R}_{\text{F}})$ in a sequential manner. In the decoding process we incorporate the valency rules of molecules by masking out impossible edges and anchor nodes. The decoding process consists of the following steps: 
\begin{compactenum}[(1)]
 \item \textbf{Anchor Node Prediction: }predict anchor nodes $a=(a_1,a_2)$ for the two fragments. These two anchor nodes are served as the binding points for linker to connect.
 \item \textbf{Node Type Prediction: }predict node type $X$ for all linker nodes.
 \item \textbf{Edge and Coordinate Prediction: }Put the two anchor nodes in a queue. Then do the following until the queue is empty:
    \begin{compactenum}[(i)]
        \item Pop a node $f$ from the queue and define it as the current \textit{focus node}. 
        \item Predict an edge between the focus node $f$ and another node. The connected node $i$ is added to the queue. If node $i$ is a linker node and is connected to the existing graph $G_t$ for the first time, predict its coordinates $\bm{r}_i$.
        \item Repeat (ii) and (iii) until an artificial stop node is connected. Update the coordinates of all nodes in the current linker. The focus node $f$ is then marked as closed, which cannot be added to the queue or connected anymore. Then go back to (i).
    \end{compactenum}
\end{compactenum}
Mathematically we factorize the joint probability into:
\begin{align}
        &p_{\theta}(G, \bm{R}|G_{\text{F}},\bm{R}_{\text{F}}, z^h, \bm{z^v})=p_{\theta}(\mathcal{E}, X, \bm{R}|\mathcal{E}_{\text{F}}, X_{\text{F}},\bm{R}_{\text{F}},z^h,\bm{z^v})\nonumber\\
     &=\underbrace{p_{\theta}(a_1, a_2|z^h, \bm{z^v})}_{\text{Anchor}}\cdot \underbrace{p_{\theta}(X|z^h)}_{\text{Node Types}}\nonumber\\
     &\cdot \underbrace{\prod_{t=0}^{T-1}p_{\theta}(\mathcal{E}_{t+1},\bm{R}_{t+1}|\mathcal{E}_t,\bm{R}_t, X, a_1, a_2, z^h,\bm{z^v})}_{\text{Edges and Coordinates}},
    \label{generative-model}
\end{align}
where $\mathcal{E}_T=\mathcal{E}$ and $\bm{R}_{T}=\bm{R}$. Details of each component are explained in the following.

\textbf{Anchor Node Prediction.}
To jointly predict two anchor nodes, we further factorize the joint probability into $p_{a_1}$, the probability of anchor $a_1$ on the first fragment $G_{\text{F},1}$, and $p_{a_2}$, the probability of anchor $a_2$ on the second fragment $G_{\text{F},2}$ conditioning on $a_1$:
\begin{equation}
\begin{split}
    & p_{\theta}(a_1,a_2|z^h, \bm{z^v}) =  p_{\theta}(a_1|\{z^h_i, \bm{z^v}_i|i\in \mathcal{V}_{\text{F}, 1})\\
    &~~~~~~~~~~~~~~~~~ \cdot  p_{\theta}  (a_2|z^h_{a_1},\bm{z^v}_{a_1}, \{z^h_i, \bm{z^v}_i|i\in \mathcal{V}_{\text{F},2}).
    \end{split}
    \label{anchor prediction}
\end{equation}
Concretely, each node on the first fragment will get a score $c_{i}=\phi_{8}(z^h_i, \norm{A_1\cdot\bm{z^v}_i})$, where $A_1 \in \mathbb{R}^{n_v \times n_v}$ is a learnable linear transformation. The scores $c_{i}$ are then passed to a softmax to compute the anchor probability for nodes of the first fragment:  $p_{a_1} = \exp(c_{a_1})/(\sum_{i\in \mathcal{V}_{\text{F},1}}\exp(c_{i}))$. Then the latent variables of this predicted anchor node as well as nodes of the second fragment are used to compute another group of scores $c^{\prime}_i = \phi_{9}(z^h_i, \norm{A_1\cdot\bm{z^v}_i}, z^h_{a_1}, \norm{A_1\cdot\bm{z^v}_{a_1}})$, and the probability of the second anchor is $p_{a_2} = \exp(c^{\prime}_{a_2})/(\sum_{i\in \mathcal{V}_{\text{F},2}}\exp(c_{i}^{\prime}))$.


\textbf{Node Type Prediction.}
Node types of linker are directly predicted using their latent variables. We leverage the self-attention mechanism~\citep{vaswani2017attention} to obtain new node features, which are then passed to an MLP to get the logits of node types. After node types are sampled, the types' embeddings are concatenated to the corresponding latent variables $z^h$ for latter procedures. 

\begin{figure}[ht!]
  \centering
  \includegraphics[page=1, scale=0.75]{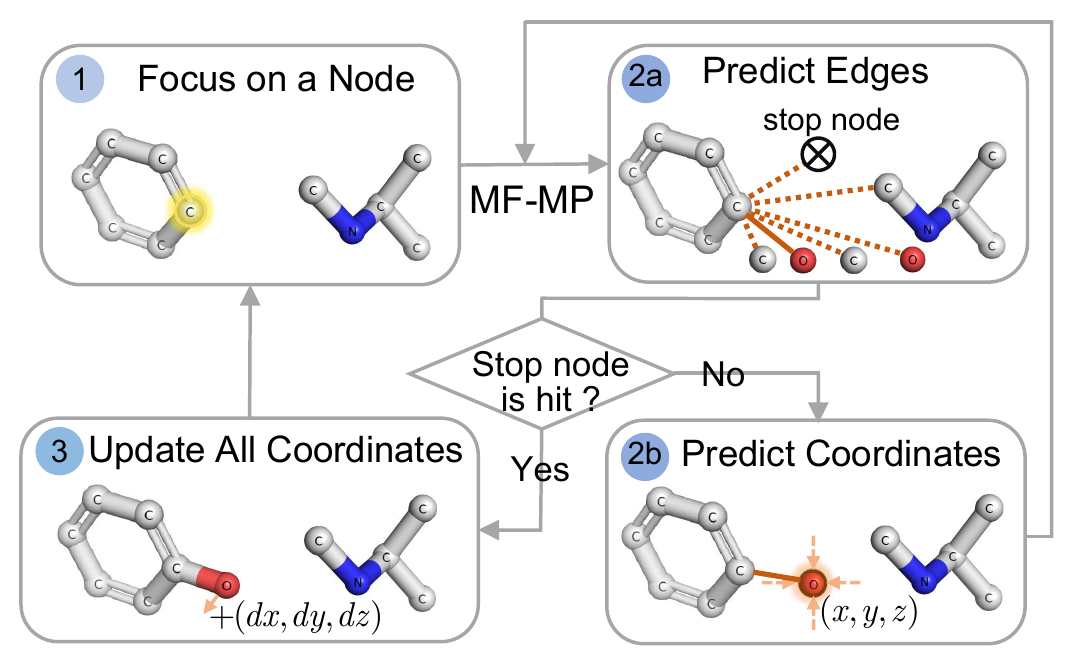}
  \caption{Illustration of sequential predictions of edges and coordinates. We first pick up on a node to focus. Then we sample an edge between focus node and other nodes (including an artifical stop node). If a linker node is first connected to the existing graph, its coordinates will be predicted. Each time before prediction MF-MP is applied to capture information from existing graph. We keep adding edges until stop node is selected, and then coordinates of all link nodes in the existing graph will be simultaneously updated. We then refocus on a new node and repeat. The procedure continues until all nodes in linker have been focused.}
  \label{decoder}
\end{figure}
\textbf{Edge and Coordinate Prediction.}
The edge and coordinate generation path $(\mathcal{E}_0,\bm{R}_0)\!\to\! (\mathcal{E}_1,\bm{R}_1)...\!\to\!(\mathcal{E},\bm{R})$ is defined by a sequence of node focusing, edge prediction and coordinate prediction procedures. The node focusing and edge connecting order is pre-determined by a Breath-First Search to enable teacher-forcing training. When a focus node $f$ is picked, we add new edges to it until connecting to the stop node. Concretely, at each step we first apply MF-MP (with different weights from that in the encoder) to obtain updated node features $\tilde{z}^h, \tilde{\bm{z}}^{\bm{v}}$ from the initial latent variables ${z}^h,{\bm{z}}^{\bm{v}}$. Then we compute the probability for edge $\mathcal{E}_{i,f}$ by 
\begin{align}
  s_{i,f}=\phi_{10}(\tilde{z}^h_i,  \tilde{z}^h_f, &\norm{A_2\cdot\tilde{\bm{z}}^{\bm{v}}_i}, \norm{A_2\cdot\tilde{\bm{z}}^{\bm{v}}_f},  \sum_{j\in \mathcal{V}}\tilde{z}^h_j, \sum_{j\in \mathcal{V}}z^h_j).\nonumber \\
  p(\mathcal{E}_{i,f}) &= \frac{\exp(s_{i,f})}{\sum_{j\in \mathcal{V}}\exp(s_{j,f})}.
\end{align}
To predict coordinates, we define  $\bm{\Omega}_i(\tilde{z}^h,\tilde{\bm{z}}^{\bm{v}}, \bm{R}_t, \bm{r})$ that outputs equivariant coordinates $\tilde{\bm{r}}_i$ for node $i$:
\begin{subequations}
\begin{align}
 p_{i,j}&=\phi_{11}(\tilde{z}^h_i, \tilde{z}^h_j, \langle A_{3}\cdot\tilde{\bm{z}}^{\bm{v}}_i, A_{4}\cdot\tilde{\bm{z}}^{\bm{v}}_j\rangle), \label{coordinates prediction-p}\\
 q_{i,j}&=\phi_{12}(\tilde{z}^h_i, \tilde{z}^h_j, \langle A_{5}\cdot\tilde{\bm{z}}^{\bm{v}}_i, A_{6}\cdot\tilde{\bm{z}}^{\bm{v}}_j\rangle),\label{coordinates prediction-q}\\
\tilde{\bm{r}}_i &= \bm{r}+\sum_{j\in \mathcal{V}_t} p_{i,j}(\bm{r}_j-\bm{r})\nonumber\\
&+\text{VN-MLP}_{7}\bigg(\sum_{j\in \mathcal{V}_t}q_{i,j}\text{VN-MLP}_{8}(\tilde{\bm{z}}^{\bm{v}}_i, \tilde{\bm{z}}^{\bm{v}}_j)\bigg) \label{coordinates predictions-Omega}
\end{align}
\label{coordinates prediction}%
\end{subequations}
for any generic coordinates $\bm{r}$ called reference point. Here $\text{VN-MLP}$ taking two inputs means concatenation along the first axis ($n_v$). The idea is to compute pair-wise interactions (\ref{coordinates prediction-p}, \ref{coordinates prediction-q}) and predict a deviation from reference point $\bm{r}$ (\ref{coordinates predictions-Omega}). If a linker node $i$ is first connected to the graph, we use the mass center of the current graph $\bar{\bm{r}}_t=\sum_{j\in \mathcal{V}_t}\bm{r}_j/\abs{\mathcal{V}_t}$ as the reference point and predict its absolute coordinates by $\bm{r}_i=\bm{\Omega}^{\text{pred}}_i(\tilde{z}^h,\tilde{\bm{z}}^{\bm{v}}, \bm{R}_t, \bar{\bm{r}}_t)$
; once the stop node is chosen, all linker nodes $i$ in the current graph will update their coordinates using their current coordinates as reference points, i.e., $\bm{r}_i^{\prime}=\bm{\Omega}^{\text{updt}}_i(\tilde{z}^h,\tilde{\bm{z}}^{\bm{v}}, \bm{R}_t, \bm{r}_i)$. Note that $\bm{\Omega}^{\text{pred}}$ and $\bm{\Omega}^{\text{updt}}$ have distinct network weights. 

\subsection{Training Using ELBO}
Variational autoencoder~\citep{kingma2013auto} is trained by maximizing the Evidence Lower Bound (ELBO):
\begin{equation}
\begin{split}
    L(\theta, \phi, \psi) &= \mathbb{E}_{z{\sim}q_{\psi}}\log p_{\theta}(G,\bm{R}|G_{\text{F}},\bm{R}_{\text{F}}, z^h, \bm{z^v}) \\
    &- \beta D_{\text{KL}}(q_{\psi}\lVert p),
    \end{split}
\end{equation}
where the prior $p(z^h,\bm{z^v}|G_{\text{F}}, \bm{R}_{\text{F}})$ simply takes standard Gaussian for all linker nodes. The reconstruction error term $\mathbb{E}_{z{\sim}q_{\psi}}\log p_{\theta}(G,\bm{R}|G_{\text{F}},\bm{R}_{\text{F}}, z^h,\bm{z^v})$ is approximated by one Monte-Carlo sampling, and we apply teacher forcing~\citep{kolen2001field} for anchor, node type and edge prediction following the pre-determined order. The loss of anchor node prediction, node type prediction and edge prediction are standard cross entropy while for loss of coordinate prediction we use log-MSE used by~\citep{yu2020tutorial}.

\subsection{Generation}
During generation, a maximum number of linker nodes is set and we sample this maximum number of latent variables $z^h,\bm{z^v}$ for linker nodes (though some nodes might never be included). Then the generation follows the same procedure as the decoder except that there is no teacher forcing.

\section{Experiments}
\label{experiment}
\begin{table*}[!ht]
    \centering
     \caption{Performance metrics for generated molecules.}
     \vspace{-8pt}
    \begin{tabular}{lcccc|cc}
    \Xhline{2.5\arrayrulewidth}
        Metrics & Valid (\%) & Recovered (\%) & Pass 2D filters (\%) & RMSD & Unique (\%) & Novel (\%) \\ \hline
        3DLinker (given anchor) & \textbf{99.20} & \textbf{94.69} & \textbf{90.35} & \textbf{0.079} & 29.24 & 32.21 \\
        3DLinker & 98.67 & 93.58 & \textbf{90.37} & \textbf{0.079} & 29.42 & 32.48 \\
        DeLinker+ConfVAE & 98.38 & 81.56 & 89.92 & 1.356 & 44.67 & 39.51 \\
        GraphAF+ConfVAE & 34.24 & 20.39 & 82.01 & 1.239 & 84.11 & \textbf{78.34} \\ 
        GraphVAE+ConfVAE  & 15.07 & 0.56 & 85.88  & 1.056 & \textbf{85.52} & 61.48 \\ 
    \Xhline{2.5\arrayrulewidth}
    \end{tabular}
    \label{exp-results-1}
\end{table*}

        

\subsection{Experiment Setup}
\textbf{Dataset.} To evaluate our model, we choose a subset of ZINC~\citep{sterling2015zinc}. For each molecule, we perform 20 times of MMFF force field optimization using RDKit~\citep{rdkit} and choose the one with the lowest energy as the ground truth. Following the same procedure from~\citep{hussain2010computationally}, the (fragments, linker) pairs are produced by enumerating all double cuts of acyclic single bonds that 
are not within any functional groups. In total, we obtain 365,749 (fragments, linker, coordinates) triplets and randomly split them into training (365,039), validation (351) and test (358). 

\textbf{Evaluation.} We evaluate the generated molecules for multiple 2D (graph) and 3D (coordinates) metrics, including the standard ones such as validity, uniqueness and novelty~\citep{brown2019guacamol}.  
In addition, we also evaluate the percentage of generated molecules passing 2D property filters, including synthetic
accessibility~\citep{ertl2009estimation}, ring aromaticity, and pan-assay interference compounds (PAINS)~\citep{baell2010new}. After filtering by validity and 2D property filters, recovery rate is calculated to report percentage of generated molecules that perfectly recover the ground truth molecule graphs. 
To evaluate quality of the 3D structures, the predicted 3D structures are compared to the ground truth using root-mean-square deviation (RMSD). Note that RMSD is only computed for generated molecules that perfectly recover the ground truth molecular graphs (including their isomorphic variants), since only recovered molecules have atom-to-atom alignment to ground truth. Following DeLinker, we compute another 3D metric, named shape-and-color similarity score ($\text{SC}_{\text{RDKit}}$).  
Appendix~\ref{appendix adapt} contains more details about the evaluation standards.

\textbf{Baselines.} Though there are works of molecular graph generative models and molecular geometry prediction given molecular graph, they rarely focus on either fragment linking or jointly modeling of both graph and geometry. Existing molecule generative models either do not work on conditional (linker) generation, or cannot predict 3D coordinates. Therefore, we implement multiple baselines by adapting multiple generative models to conditional generative models to generate 2D linker graphs given two molecular fragments. The generated graphs are then taken by a molecular geometry prediction model, ConfVAE~\citep{xu2021end}, to predict the 3D coordinates of each atom. Our baselines include DeLinker$+$ConfVAE, GraphAF$+$ConfVAE and GraphVAE$+$ConfVAE. DeLinker is an existing baseline for conditional linker generation. GraphAF is an autoregressive flow model, and GraphVAE is a VAE-based model with graph-level encodings. The latter two are adapted to conditional graph generation. For ConfVAE, we modify its decoder flow model to conditionally predict linker coordinates. Please see appendix \ref{appendix adapt} for implementation details. 


\begin{table}[t]
\resizebox{0.5\textwidth}{!}{
  \begin{minipage}[t]{0.55\textwidth}
    \centering
    \caption{$\text{SC}_{\text{RDKit}}$ score distribution ($\%$) and averaged score.}
    \vspace{-8pt}
    \begin{tabular}{lcccc}
    \Xhline{2.5\arrayrulewidth}
        Metrics  &  \multicolumn{4}{c}{$\text{SC}_{\text{RDKit}}$  Fragments}  \\
        ~&$>0.7$ & $>0.8$ & $>0.9$ & Average \\ \hline
        3DLinker (given anchor) & \textbf{43.10} & \textbf{16.09} & \textbf{2.60} & \textbf{0.684}\\
        3DLinker&  42.55 & 15.85 & 2.49 & 0.683\\
        DeLinker+ConfVAE &39.96&13.39&1.93&0.675 \\
        GraphAF+ConfVAE &19.33&3.36&0.32&0.624\\ 
        GraphVAE+ConfVAE &13.17 &2.15 &0.00 &0.601\\
    \Xhline{2.5\arrayrulewidth}
    \end{tabular}
    \label{exp-results-2}
\end{minipage}
}
\end{table}

\subsection{Results}
We trained 3DLinker for 20 epochs using Adam optimizer with  learning rate 0.006, batch size 48 and KL trade-off $\beta=0.6$. Training details for other baselines are included in Appendix \ref{appendix adapt}. Each model generates 250 samples per fragments, which leads to in total $250\times 358=89500$ samples. Note that since DeLinker takes anchor nodes as known ground truth, we add another comparison where anchor nodes are known to 3DLinker when generation, denoted as 3DLinker (given anchor). Results in Table \ref{exp-results-1} and Table \ref{exp-results-2} show that 3DLinker could generate valid and similar linker graph structures with a higher recovery rate, and at the same time achieve accurate predictions of the 3D coordinates of each atom. An interesting observation is that although focusing on 3D structures, 3DLinker achieves superior recovering accuracy of the 2D molecular graphs. Our interpretation is that incorporating 3D constraints benefits to the reconstruction of 2D graphs, though it might influence the diversity (low novelty and uniqueness). This also explains why the novelty and uniqueness are relatively low because the 3D constraints significantly reduce the valid chemical compound structures. In addition, GraphAF and GraphVAE are not able to obey valency rules during training, which explains their low valid and recovery rates. In terms of 3D structure predictions, the low RMSD of 3DLinker demonstrates the effectiveness of both equivariant features and coordinate update strategies. ConfVAE performs poorly since the error of distance matrix prediction scales up quickly with the number of nodes. 

\begin{figure*}[!ht]
  \centering
  \includegraphics[page=1, scale=0.5]{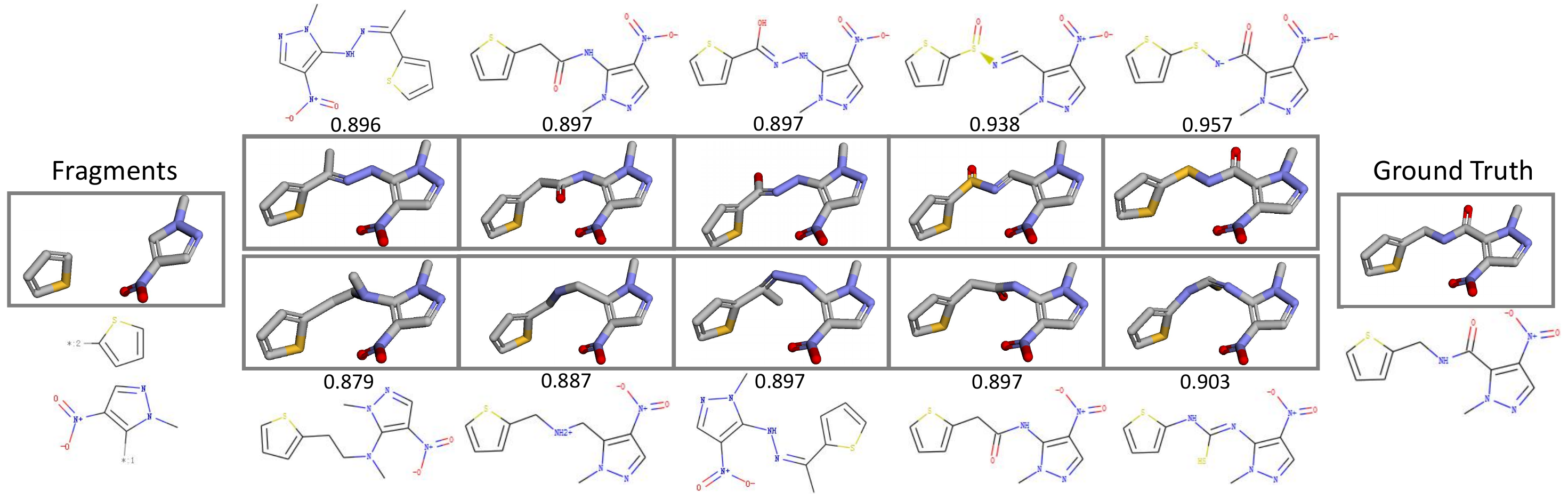}
  \vspace{-10pt}
  \caption{An example of fragment linking. The top-5 similar by $\text{SC}_{\text{RDKit}}$ Fragments proposed by 3DLinker (first row) and DeLinker+ConfVAE (second row) are shown. Generations from 3DLinker are more realistic and similar to ground truth in terms of $\text{SC}_{\text{RDKit}}$ and 3D geometry.}
  \label{case study}
\end{figure*}

\subsection{Ablation Study}
In the ablation study, we focus on two key components of 3DLinker: (1) equivariant features vs. invariant features alone; (2) update of all coordinates until the generation finishes, instead of fixing them in the 3D space one by one. Results display a significant decrease of performance after removing either equivariant features or coordinate update, especially for the RMSD and recovery rate. See Appendix \ref{ablation} for details.

\subsection{Molecular Property Prediction}
We show a downstream task of molecule property prediction. We use the learned latent variables from VAE models to predict the Quantitative Estimate
of Drug-Likeness (QED) by a gated sum:
\begin{equation}
    \text{QED} = \sigma\left(\sum_{i\in G}\sigma(\phi_{13}(z^h_i, \norm{V\bm{z^v}_i}))\phi_{14}(z^h_i, \norm{U\bm{z^v}_i})\right),
\end{equation}
where $\phi_{13}$ and $\phi_{14}$ are two separate neural networks, $V,U$ are two linear transform matrices and $\sigma(\cdot)$ is a sigmoid function. QED is widely adopted to quantify the potential for a small molecular to be a drug while the function of molecular linkers is to connect two existing drugs. Here we adopt QED to check whether 3DLinker produces biological and biochemical meaningful molecules.

\begin{table}[t]
\resizebox{0.5\textwidth}{!}{
  \begin{minipage}[t]{0.5\textwidth}
    \centering
    \caption{Root Mean Squared Error of QED prediction. }
    \begin{tabular}{lccccc}
    \Xhline{2.5\arrayrulewidth}
        Metrics  &  $\text{RMSE}_{\text{QED}}$\\ \hline
        3DLinker&  0.0833  &\\
        DeLinker+ConfVAE &0.1077\\
        GraphVAE+ConfVAE &0.1179\\
    \Xhline{1\arrayrulewidth}
    \end{tabular}
    \label{qed prediction}
\end{minipage}
}
\end{table}

As shown in Table \ref{qed prediction}, 3DLinker achieves the lowest Root Mean Square Error (RMSE) among all the models, suggesting a good expressing power of its learned latent representations.

\subsection{Visualization}
In the end, we present multiple examples of linker design by visualizations. In Figure \ref{case study}, we show the top-5 molecules with highest $\text{SC}_{\text{RDKit}}$ generated by 3DLinker (first row in the middle) and DeLinker+ConfVAE (second row in the middle). It is obvious that molecules from 3DLinker are generally more similar to the ground truth 2D chemical graph, and have a better spatial alignment with the ground truth 3D structure. 

\section{Conclusions}
\label{conclusions}
We developed 3DLinker, a conditional variational autoencoder that is able to jointly models graph and 3D representations, predict anchor nodes and samples linkers. Experiments show that 3DLinker is able to generate linkers with both high recovery rate and precise geometry.


There are still limitations to be considered in the future. First, models should be able to sample number of linker nodes instead of setting a maximal number of nodes in advance. Second, though it is known that spatial configuration of fragments should not be disturbed, in practice slight difference exist between different linkers, which is an avenue for future works to take into account.


\section*{Software and Data}
Codes and data are included in the supplementary materials.






\newpage
\onecolumn
\appendix
\section{Proof of Invariance/Equivariance}
\label{appendix equivariance}
In this section we prove that 3DLinker satisfies E(3) equivariance, namely for all $g\in \text{E}(3)$, we have $p_{\theta}(G,\pi(g)\bm{R}|G_{\text{F}},\pi(g)\bm{R}_{\text{F}})=p_{\theta}(G,\bm{R}|G_\text{F},\bm{R}_{\text{F}})$. For simplicity, we denote E(3) invariance and equivariance as E(3)-inv and E(3)-eqv, and O(3) invariance and equivariance as O(3)-inv and O(3)-eqv, and translational invariance as T-inv. 

We first prove $\tilde{h}$ is E(3)-inv while $\tilde{\bm{v}}$ is O(3)-eqv/T-inv after MF-MP (\ref{mixed features}, \ref{message function}, \ref{aggregation function})
\begin{lemma}
\label{my-lemma}
    In MF-MP, $\tilde{h}$ is E(3)-inv and $\tilde{\bm{v}}$ is O(3)-eqv, if $h$ is E(3)-inv and $\bm{v}$ is O(3)-eqv.
\end{lemma}
\begin{proof}
    First let us show that in (\ref{mixed features}), $h^{\prime},h^{\prime\prime}$ are E(3)-inv and $\bm{v}^{\prime}$ is O(3)-eqv/T-inv. Note that $\text{VN-MLP}(\bm{v})$ is E(3)-eqv and T-inv due to the properties of vector neurons, and thus the norm $\norm{\text{VN-MLP}(\bm{v})}$ is O(3)-inv and T-inv (because O(3) preserves inner product), namely E(3)-inv. Therefore in (\ref{eqn-a}, \ref{eqn-b}) $h^{\prime}$ and $h^{\prime\prime}$ are both E(3)-inv. Equation (\ref{eqn-c}) is essentially an O(3)-eqv $\text{VN-MLP}_{hv}(\bm{v}_j)$ scaled by E(3)-inv features $\text{diag}\{\phi_{hv}(h_j)\}$, and thus its output $\bm{v}^{\prime}_j$ is also O(3)-eqv/T-inv. 
    
    Then in message function (\ref{eqn-d}, \ref{eqn-e}), message $m^h$ is E(3)-inv since kernel $\text{Ker}$ is only a function of E(3)-inv distance $\norm{\bm{r}_{i,j}}$. For (\ref{eqn-e}), The first term of RHS is a O(3)-eqv/T-inv features $\bm{v}^{\prime}$ scaled by E(3)-inv kernel, which makes it O(3)-eqv/T-inv. Similarly for the second term is a O(3)-eqv/T-inv relative displacement $\bm{r}_{i,j}$ scaled by both E(3)-inv kernels and features. Therefore $m^h$ is E(3)-inv and $\bm{m^v}$ is O(3)-eqv/T-inv.
    
    Finally, in (\ref{agg-h}) $\tilde{h}$ is E(3)-inv since all its inputs are E(3)-inv. In (\ref{agg-v}) $\tilde{\bm{v}}$ is O(3)-eqv/T-inv due to vector neurons. 
\end{proof}

\begin{theorem}
The generative model $p_{\theta}(G,\bm{R}|G_{\text{F}},\bm{R}_{\text{F}})$ (\ref{encoder}, \ref{encoder-frag}, \ref{generative-model}) satisfies equivariance condition (\ref{eqv-condition}). 
\end{theorem}
\begin{proof}
    By lemma \ref{my-lemma}, the encoded latent variables $z^h$ and $\bm{z^v}$ from (\ref{encoder}, \ref{encoder-frag}) is E(3)-inv and O(3)-eqv/T-inv respectively. In the decoding process, anchor nodes, node types and edges are all predicted from $z^h$ or norm of $\bm{z^v}$. Thus the probability of graph $G$ is E(3)-inv. The coordinates are predicted through (\ref{coordinates prediction}). Note that $p_{i,j}$, $q_{i,j}$ in (\ref{coordinates prediction-p}, \ref{coordinates prediction-q}) are E(3)-inv and $\bm{r}$ in (\ref{coordinates predictions-Omega}) is E(3)-eqv (mass center), which implies terms $\sum_{j\in \mathcal{V}_t} p_{i,j}(\bm{r}_j-\bm{r})$ and $\text{VN-MLP}_{7}\left(\sum_{j\in \mathcal{V}_t}q_{i,j}\text{VN-MLP}_{8}(\tilde{\bm{z}}^{\bm{v}}_i, \tilde{\bm{z}}^{\bm{v}}_j)\right)$ are all O(3)-eqv/T-inv. Finally we can conclude that $\bm{\Omega}_i$ is E(3)-eqv because an O(3)-eqv/T-inv quantity pluses an E(3)-eqv quantity $\bm{r}$ results in an E(3)-eqv quantity). Therefore $\bm{R}$ is E(3)-eqv.
\end{proof}

\section{Point Convolution}
\label{point convolution}
In the message function~(\ref{message function}), kernel $\text{Ker}(\norm{\bm{r}_{i,j}})$ assigns weights relying on distance $\norm{\bm{r}_{i,j}}$. Concretely in our implementation, our kernel first applies  Gaussian functions with $10$ different means and perform a learnable affine transform: 
\begin{equation}
    \text{Gaussian}: \norm{\bm{r}_{i,j}}\mapsto \begin{pmatrix}
    \exp\{-k(\norm{\bm{r}_{i,j}} - \mu_1)\}\\
    \exp\{-k(\norm{\bm{r}_{i,j}} - \mu_2)\}\\
    ...\\
    \exp\{-k(\norm{\bm{r}_{i,j}} - \mu_{10})\}\\
    \end{pmatrix},\quad \text{Ker}(\norm{\bm{r}_{i,j}}) = \text{Linear}(\text{Gaussian}(\norm{\bm{r}_{i,j}}))+\text{Bias}.
\end{equation}
where $\mu_1<\mu_2<...<\mu_{10}$ are hyper-parameters and $k$ is a learnable parameter. The intuition behind is to capture the intensity of interactions between nodes (atoms) with different distances. The largest mean $\mu_{10}$ is basically the maximal correlation length: if two nodes are separated by distance beyond $\mu_{10}$, their interaction is nearly neglectable.

Mathematically, message functions~(\ref{message function}) can be seen as tensor products using spherical harmonics, as described in Tensor Field Networks~\citep{thomas2018tensor}. In group representation theory, invariant features $h$ and equivariant features $\bm{v}$ are called type-0 and type-1 tensors respectively, and the theory~\citep{griffiths2018introduction} tells us the correct way to construct new type-0 features $\tilde{h}$ or type-1 features $\tilde{\bm{v}}$ using type-$l$ spherical harmonics $Y^{l}(\bm{r})$ and $h,\bm{v}$. In our case since we only have type-0 and type-1 features, $Y^{0}(\bm{r})\propto 1$ and $Y^{1}\propto\bm{r}$ are all we need, which explains the design of~(\ref{message function}).

Note that there are several differences between our method and tensor products in Tensor Field Networks (TFN). First TFN is based on SO(3) equivariant, while we seek for O(3) equivariant. So terms like cross product are discarded since they violate mirror symmetry. Besides, we mix different types of features before convolution, in contrast to convolution on raw features. Finally, TFN uses simple activation functions like scaling with norm, while we leverage Vector Neuron for novel non-linearity.

\section{Experiment Details}
\label{appendix adapt}

\textbf{Some evaluation metrics.}
Validity is defined by percentage of generated molecules that both obey chemical constraints (valency) and successfully links two fragments into connected graphs. Invalid molecules are discarded for the following evaluations. Uniqueness means percentage of non-duplicate generated molecules: $\frac{\text{len}(\text{Set}(\text{generated molecules}))}{\text{len}(\text{generated molecules})}$. Novelty refers to percentage of generated molecules whose linkers are not present in training set. 

$\text{SC}_{\text{RDKit}}$ uses two RDKit built-in functions as described in~\citep{putta2005conformation} and ~\citep{landrum2006feature} to compute color similarity scores between 
two 3D molecules based on the overlap of
their pharmacophoric features. And the shape similarity
score is a simple volumetric comparison between the two
3D molecules. Both scores are between 0 (no match) and
1 (perfect match), which are averaged to produce a final score
between 0 and 1. Scores above 0.7 indicate a good match,
while scores above 0.9 suggest an almost perfect match.
Following DeLinker, $\text{SC}_{\text{RDKit}}$ is measured only on fragments (we re-generate their coordinates together with the linker), which embodies the capability to generate linkers without disturbing fragments. 

\textbf{3DLinker.} We train 3DLinker for 20 epochs with kl trade-off beta 0.6. Note that the anchor node prediction~\ref{anchor prediction} is asymmetric to the permutation of $a_1$ and $a_2$, and thus we apply a permutation of two fragments to enhance our model.

\textbf{GraphAF.}
Originally GraphAF is an autogressive flow model $p(G)=\prod_t p(G_{t+1}|G_t)$. To model a conditional probability $p(G|G_{\text{F}})$, all we need is to mask out loss of $G_{\text{F}}$ and only compute loss starting from $G_{\text{F}}$ to $G$. We trained GraphAF for 170 epochs, and other hyper-parameters are consistent with its source code (\href{https://github.com/DeepGraphLearning/GraphAF}{https://github.com/DeepGraphLearning/GraphAF}).

\textbf{GraphVAE.}
GraphVAE represents a graph $G$ as node types $F$, adjacency $A$ and edge types $\mathcal{E}$ and maps them into a graph-level representation $z$. Then $z$ is decoded into $\tilde{F},\tilde{A},\tilde{\mathcal{E}}$ with an additional graph matching to compute loss. To modify it into a conditional generative models, a naive approach is to build a generative model for linker $G_{\text{L}}$ only, and then predict the anchor nodes that connects to fragments. Concretely, let $a_1,a_2$ be anchor nodes of fragments, and $b_1, b_2$ be the corresponding anchor nodes of linker, the decoder model $p(G|G_{\text{F}},z)$ is:
\begin{equation}
    p(G|G_{\text{F}}, z) = p(G_{\text{L}}|z_{\text{F}}, z)p(a_1,a_2,b_1,b_2|z_{\text{F}}, z)p(\mathcal{E}_{a_1,b_1},\mathcal{E}_{a_2,b_2}|z_{\text{F}}, z),
\end{equation}
where $z_F$ is a graph-level encoding of fragments and $\mathcal{E}_{a_1,b_1},\mathcal{E}_{a_2,b_2}$ are two edges that connects fragments and linker. The realization is based on code provided by \href{https://github.com/snap-stanford/GraphRNN}{https://github.com/snap-stanford/GraphRNN}. In experiments, all hyper-parameters are unchanged except KL trade-off beta is $0.6$.

\textbf{ConfVAE.}
ConfVAE is a VAE for geometry generation given graph $p(\bm{R}|G)$. It encodes graph $G$ and distance matrix $D=\{d_{i,j}=\norm{\bm{r}_i-\bm{r}_j}|i,j\in G\}$ into latent variables $z$, and uses a flow model $f$ to decode distance matrix $D=f(G,\bm{z})$. Concretely, its flow $f$ is 
\begin{equation}
    D = f(G,z) = D(0) + \int_{0}^{t}g_{\theta}(G, D(\tau), z)d\tau,
\end{equation}
where $g_{\theta}$ is a Message Passing Neural Networks (MPNN) and $D(0)\sim N(0,I)$ is a base distribution. To transfer $p(\bm{R}|G)$ to a conditional model $p(\bm{R}|G,\bm{R}_0)$, we modify the flow to
\begin{equation}
    \begin{pmatrix}
    D_0 \\ D_{L}
    \end{pmatrix} = f(G, z, D_0) = \begin{pmatrix}
    D_0 \\ D_{L}(0)
    \end{pmatrix} + \int_{0}^{t}\begin{pmatrix}
    0 & 0 \\ 0 & 1
    \end{pmatrix}g_{\theta}(G, D_0, D_L(\tau), z)d\tau
    \label{first opt}
\end{equation}
where $D_0=\{d_{i,j}=\norm{\bm{r}_i-\bm{r}_j}|i,j\in G_0\}$ are distances we already knew while $D_L=\{d_{i,j}=\norm{\bm{r}_i-\bm{r}_j}|i\in G-G_0 \text{ or }j\in G-G_0\}$ are distances between nodes with at least one is in linker. After distance matrix $D$ is predicted, we transform it into absolute coordinates by optimizing the following:
\begin{equation}
    \min_{\{\bm{r}_i|i\in G-G_0\}} \sum_{i,j\in G}(D_{i,j}-\norm{\bm{r}_i-\bm{r}_j})^2.
    \label{second opt}
\end{equation}
Note that we only need to optimize $\{\bm{r}_i|i\in G-G_0\}$ since $\bm{R}_0=\{\bm{r}_i|i\in G_0\}$ are given. Also there is no need for coordinates alignment since coordinate system is well-defined by coordinates of fragments $\bm{R}_0$. Code is provided by \href{https://github.com/MinkaiXu/ConfVAE-ICML21}{https://github.com/MinkaiXu/ConfVAE-ICML21}. Note that although ConfVAE can be trained in an end-to-end manner (both equation \ref{first opt} and \ref{second opt}), it only makes a little improvements ($0.01$) compared to training the flow alone (see the experiments of its orignal paper \cite{xu2021end}). Thus we choose only to train the flow model alone in our experiments. For each graph, we sample one geometry which is optimized by (\ref{second opt}) with 10 times random initialization and 300 steps gradient descent.

\section{Ablation Study}
We conduct ablation study on two aspects: removing equivariant features and coordinates update strategy. We train these three models for 20 epochs and evaluate them by the same methods in previous experiments. Results are shown in Table \ref{ab-results-1} and \ref{ab-results-2}. 
\label{ablation}
\begin{table*}[!ht]
    \centering
    \caption{Ablation Study. (eqv-) stands for removing equivariant features while (update-) means removing coordinates update strategy.}
      \vspace{-8pt}
    \begin{tabular}{lcccc|cc}
    \Xhline{2.5\arrayrulewidth}
        Metrics & Valid (\%) & Recovered (\%) & Pass 2D filters (\%) & RMSD & Unique (\%) & novel (\%) \\ \hline
        3DLinker & 98.67 & 93.58 & 90.37 & 0.079 & 29.42 & 32.48 \\
         3DLinker (eqv-)  & 99.42 & 86.59 & 92.68 & 1.352 & 34.58 & 27.02 \\ 
        3DLinker (update-)  & 98.85 & 39.94 & 62.81 & 0.399 & 55.93 & 72.25 \\ 
    \Xhline{2.5\arrayrulewidth}
    \end{tabular}
    \label{ab-results-1}
\end{table*}
\begin{table*}[!ht]
    \centering
    \caption{Ablation study. (eqv-) stands for removing equivariant features while (update-) means removing coordinates update strategy.}
    \vspace{-8pt}
    \begin{tabular}{lcccc}
    \Xhline{2.5\arrayrulewidth}
        Metrics  &  \multicolumn{4}{c}{$\text{SC}_{\text{RDKit}}$  Fragments}  \\
        ~&$>0.7$ (\%) & $>0.8$ (\%) & $>0.9$ (\%)& Average \\ \hline
        3DLinker &  42.55 & 15.85 & 2.49 & 0.683 \\
        3DLinker (eqv-)  &38.51&13.15&1.76&0.672 \\ 
        3DLinker (update-)  &37.34 & 10.87& 1.13& 0.670 \\
    \Xhline{2.5\arrayrulewidth}
    \end{tabular}
    \label{ab-results-2}
\end{table*}
We can see a dramatic drop of performances after moving either equivariant features or coordiantes update. Especially, both these two modules contribute greatly to the prediction of coordinates, leading to a decrease of RMSD by 1.3 and 0.3 respectively. Also it is interesting to see that coordinates update has a significant impact on graph quality. A possible reason is that updating the coordinates results in a flexible intermediate coordinates, which may increase the expressive capacity of features. In some sense it is similar to EGNN~\citep{satorras2021n}, who also update intermediate coordinates in the forward pass.
\end{document}